\documentclass{article}
\usepackage{ijcai25}

\usepackage{times}
\usepackage{soul}
\usepackage{url}
\usepackage[hidelinks]{hyperref}
\usepackage[utf8]{inputenc}
\usepackage[small]{caption}
\usepackage{graphicx}
\usepackage{amsmath}
\usepackage{amsthm}
\usepackage{booktabs}
\usepackage{algorithm}
\usepackage{algorithmic}
\usepackage[switch]{lineno}

\newtheorem{theorem}{Theorem}

\usepackage{natbib} 
\usepackage{dirtytalk}
\usepackage{multirow}
\usepackage{boldline}
\usepackage{adjustbox}
\usepackage{balance}
\usepackage{caption}
\usepackage{subcaption}

\input{Definitions}

\title{Pretrained Diffusion Models Are Inherently Skipped-Step Samplers}

\author{
    Wenju Xu
    \emails
    xuwenju123@gmail.com
}

\begin{document}

\maketitle

\begin{abstract}
Diffusion models have been achieving state-of-the-art results across various generation tasks. However, a notable drawback is their sequential generation process, requiring long-sequence step-by-step generation. Existing methods, such as DDIM, attempt to reduce sampling steps by constructing a class of non-Markovian diffusion processes that maintain the same training objective. However, there remains a gap in understanding whether the original diffusion process can achieve the same efficiency without resorting to non-Markovian processes. In this paper, we provide a confirmative answer and introduce skipped-step sampling, a mechanism that bypasses multiple intermediate denoising steps in the iterative generation process, in contrast with the traditional step-by-step refinement of standard diffusion inference. Crucially, we demonstrate that this skipped-step sampling mechanism is derived from the same training objective as the standard diffusion model, indicating that accelerated sampling via skipped-step sampling via a Markovian way is an intrinsic property of pretrained diffusion models. Additionally, we propose an enhanced generation method by integrating our accelerated sampling technique with DDIM. Extensive experiments on popular pretrained diffusion models, including the OpenAI ADM, Stable Diffusion, and Open Sora models, show that our method achieves high-quality generation with significantly reduced sampling steps.
\end{abstract}


\section{Introduction}
Diffusion models have emerged as the state-of-the-art approach in a wide range of generative tasks, including image generation \cite{dhariwal2021diffusion, rombach2022high, zhai2023feature, xu2021drb,xu2023learning}, video generation \cite{opensora}, and multi-modal generation \cite{Nair2024MaxFusionPM, Chen2024, dong2021dual}. Despite their impressive performance, one persistent challenge associated with diffusion models is the high latency during the inference phase. Typically, diffusion sampling methods can be divided into two categories: SDE (Stochastic Differential Equation)-based methods and ODE (Ordinary Differential Equation)-based methods. SDE-based methods incorporate a noise term to maintain the stochastic nature of the diffusion process, with the reverse-time SDE capturing the gradual denoising of the data. This approach often necessitates solving a reverse-time SDE, as exemplified by Denoising Diffusion Probabilistic Models (DDPM) \cite{ho2020denoising}. On the other hand, ODE-based methods reformulate the diffusion process into a deterministic one by exploiting the relationship between the forward and reverse processes. This typically involves solving a reverse-time ODE that traces the data's trajectory from noise back to the original data distribution, as seen in Denoising Diffusion Implicit Models (DDIM) \cite{song2020denoising}. Other prominent ODE-based approaches include but are not limited to score-based models \cite{Song2020ScoreBasedGM}, Probability Flow ODEs \cite{Lu2022DPMSolverAF, Xue2024AcceleratingDS}, and Pseudo-Non-Markovian Sampling (PMLS) \cite{Liu2022PseudoNM}.

% \begin{figure}[t!]
% %\vspace{-0.3cm}
% \centering     
% \includegraphics[width=\linewidth]{figures/overview.pdf}
% % \vspace{-0.3cm}
% \caption{Overview of our proposed skipped-step diffusion sampler. Our method can denoise a noisy input to arbitrary steps from any pretrained diffusion models.}
% \label{fig:skipped_step}
% % \vspace{-0.4cm}
% \end{figure}

\begin{figure}[t!]
    \centering
    \begin{subfigure}[t]{0.5\textwidth}
        \centering
\includegraphics[width=\textwidth]{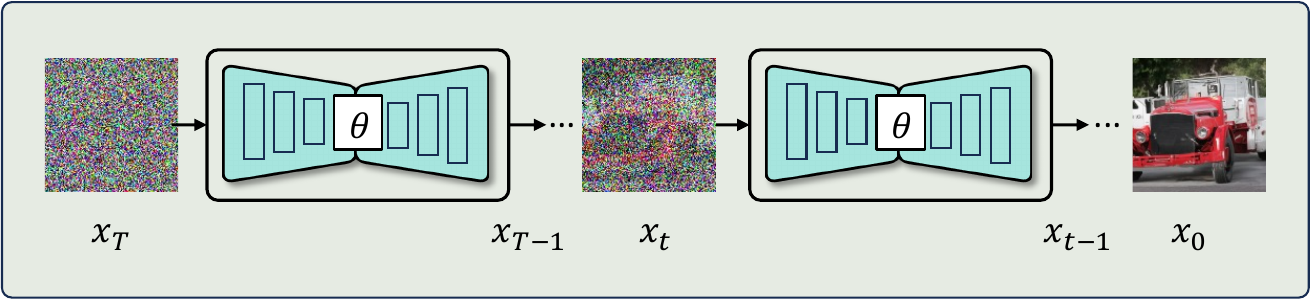}
        \caption{Standard Sampling}
    \end{subfigure}
    \begin{subfigure}[t]{0.5\textwidth}
        \centering
\includegraphics[width=\textwidth]{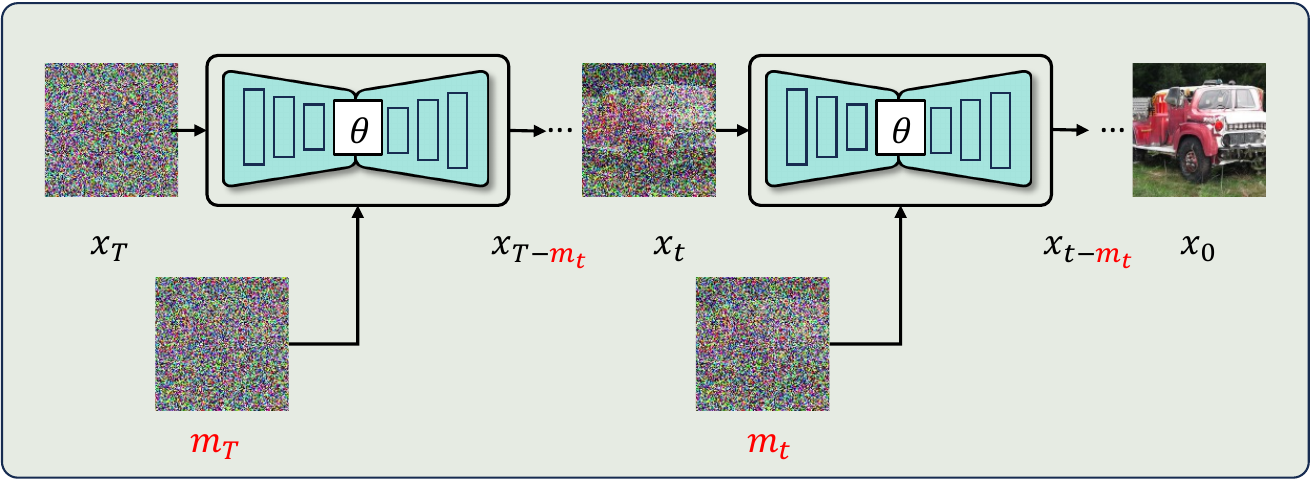}
        \caption{Skipped-Step Sampling}
    \end{subfigure}
    \caption{Overview of our proposed skipped-step diffusion sampler. Our method in principle can denoise a noisy input by skipping many intermediate denoising steps from any pretrained diffusion models. The principle is derived from the same objective function as the standard diffusion training, indicating it an inherent property of pretrained diffusion models.}
\label{fig:skipped_step}
\end{figure}

Significant efforts have been devoted to accelerating diffusion inference, with most advancements focusing on ODE-based methods. These include the development of more efficient numerical methods \cite{Xue2024AcceleratingDS} and the design of new diffusion processes that retain the same training objective \cite{song2020denoising}. DDIM, for example, introduces a class of non-Markovian diffusion processes that can achieve deterministic denoising using a subset of time steps. However, DDPM, one of the leading methods that can outperform DDIM in many cases (especially when given enough refinement iterations), still adheres to the standard iterative denoising process, often requiring a substantial number of iterations (e.g., 1000). This raises an important question: Can we uncover an intrinsic property within the DDPM framework that facilitates more efficient sampling? How can we design SDE-based sampling methods that enhance efficiency while maintaining the effectiveness and robustness of DDPM?

In this paper, we address this gap by proving that the standard diffusion model's objective is inherently equivalent to a variant that supports skipped-step sampling during generation. We define "skipped-step" sampling as the ability to denoise an input that is equivalent to performing multiple denoising steps in the standard diffusion model, for example, generating $\xb_t$ directly from $\xb_{t+\ell}$ where $\ell$ is an arbitrary integer. This discovery allows us to accelerate diffusion model inference by bypassing many of the intermediate steps typically required in DDPM, leveraging an intrinsic property of pretrained diffusion models. The overview of our proposed method is illustrated in Figure~\ref{fig:skipped_step}. To further enhance our approach, we propose an improved version to integrate it with DDIM, where we use our method to efficiently generate an initial noisy input, which is then refined by DDIM. This combination allows us to harness the strengths of both methods. Experimental results on various generation tasks, including image generation with the OpenAI ADM model \cite{dhariwal2021diffusion}, high-resolution image generation with the stable diffusion model \cite{rombach2022high, wang2025domain}, and video generation with the open Sora model \cite{opensora}, demonstrate significant improvements in both generation quality and efficiency.

Our contributions are threefold:
\begin{itemize}
\item We identify an intrinsic property of pretrained diffusion models that enables skipped-step sampling during diffusion generation.
\item We propose a novel skipped-step sampling mechanism to accelerate diffusion model inference, and we further enhance this approach by integrating it with DDIM.
\item We conduct extensive experiments across various generation tasks, including image and video generation using state-of-the-art pretrained diffusion models, showing consistent improvements in generation efficiency and quality over existing baselines.
\end{itemize}

\section{Pretrained Diffusion Models as Skipped-Step Samplers}
In this section, we introduce our proposed skipped-step sampling technique for diffusion models, leveraging arbitrary pretrained models. For consistency, we use \( t \) to denote the time step, and \(\xb_t\) represents the data at time \( t \).

\subsection{Diffusion Models and DDPM Sampling}
Diffusion models are a class of generative models that progressively transform noise into structured data through a sequence of stochastic steps. These models consist of a forward process \( q(\xb_t) \) and a parameterized reverse process \( p_{\thetab}(\xb_{t-1}|\xb_t) \), where \( t \) indexes the time step. The forward process gradually introduces Gaussian noise to the data over \( T \) time steps, defined as:
\[
    q(\xb_t | \xb_{t-1}) = \mathcal{N}(\xb_t; \sqrt{\alpha_t} \xb_{t-1}, (1-\alpha_t)\Ib)
\]
where \( \alpha_t \) is a user-defined noise schedule parameter.

The reverse process, which is parameterized, aims to denoise the data and is modeled as:
\[
    p_\thetab(\xb_{t-1} | \xb_t) = \mathcal{N}(\xb_{t-1}; \mu_\thetab(\xb_t, t), \Sigma_\thetab(\xb_t, t))~,
\]
In practice, we usually set $\Sigma_\thetab(\xb_t, t) = \sigma_t\Ib$ with $\{\sigma_t\}$ a user-defined sequence for simplicity. The training objective is to learn the parameters \( \thetab \) by minimizing a variant of the variational bound on the negative log-likelihood:
\begin{align}\label{eq:obj}
    L(\thetab) &= \mathbb{E}_q \left[ \sum_{t=1}^T D_{KL}(q(\xb_{t-1}|\xb_t, \xb_0) \| p_\thetab(\xb_{t-1}|\xb_t)) \right. \nonumber \\
    & \left. + D_{KL}(q(\xb_{T}|\xb_0) \| p(\xb_{T})) - \log p_{\thetab}(\xb_0|\xb_1) \right]
\end{align}
Sampling from the model involves reversing the diffusion process, starting from Gaussian noise \( \xb_T \) and iteratively applying the learned reverse process as outlined in Algorithm~\ref{alg:dmsample}.

\begin{algorithm}[H]
\caption{Standard Diffusion Model Sampling}
\label{alg:dmsample}
\textbf{Input}: A pretrained diffusion model $\thetab$ \\
\textbf{Output}: A generated sample
\begin{algorithmic}[1] %[1] enables line numbers
\STATE Initialize \(\xb_T \sim \mathcal{N}(0, \Ib)\); Set $t = T$.
\WHILE{$t > 0$}
\STATE Sample \(\epsilon \sim \mathcal{N}(0, \Ib)\) if \( t > 1 \), else set \(\epsilon = 0\).
\STATE Compute \(\xb_{t-1} = \frac{1}{\sqrt{\alpha_t}}\left(\xb_t - \frac{1 - \alpha_t}{\sqrt{1 - \bar{\alpha}_t}}\epsilon_{\thetab}(\xb_t, t)\right) + \sigma_t\epsilon\).
\STATE Set $t = t-1$.
\ENDWHILE
\STATE \textbf{return} \(\xb_0\)
\end{algorithmic}
\end{algorithm}

This framework enables diffusion models to generate high-quality samples by effectively learning to denoise step by step, utilizing both forward and reverse stochastic processes. However, as mentioned earlier, this step-by-step sampling procedure results in significant latency during generation. In the subsequent sections, we present our method that facilitates skipped-step sampling in diffusion generation, thereby accelerating the inference process.

\subsection{Skipped-Step Sampling in Diffusion Models}
To enable skipped-step sampling, we propose to add a multi-step noise prediction into the original diffusion model as an additional task in training. Specifically, in the forward process, we define the following multi-step noising operation at time $t$ as: $q(x_t|x_{t-m})$ for $m$ randomly sampled from $[1, t]$. Our idea is to learn the multi-step noise in the reversed model. Once achieving this, we will be able to perform skipped-step sampling in generation via the reverse model. To this end, we first formally define skipped-step sampling in Definition~\ref{def:skip} below.
\begin{definition}[Skipped-step sampling]\label{def:skip}
    Skipped-step sampling is an accelerated sampling method for diffusion model generation. The method starts from a noise input $\xb_T$, and at each iteration $t$, generates a denoise version of the input from a pretrained reverse model $\thetab$ by directly skipping $m$ intermediate steps, {\it e.g.}, $\xb_{t-m}\sim p_{\thetab}(\xb_{t-m}|\xb_t)$, until $\xb_0$. Here $m$ is an arbitrary integer satisfying $t-m\geq 0$.
\end{definition}
Based on the definition, the key is to learn the skipped-step reversed model $p_{\thetab}(\xb_{t-m}|\xb_t)$. Interestingly, we can prove that this reverse model is equivalent to the reverse model trained with standard diffusion models. Consequently, one can simply re-use a pretrained diffusion model for skipped-step sampling for free.

To prove this, we need to first define the corresponding forward noising distribution $q(\xb_{t}|\xb_{t-m})$, based on which the corresponding forward posterior distribution $q(\xb_{t-m}|\xb_t, \xb_0)$ is calculated to match the reversed model $p_{\thetab}(\xb_{t-m}|\xb_t)$. Here $m$ is an integer such that $t-m \geq 0$. Following standard diffusion model setup, with a noise schedule parameter $\alpha_t$, we define $\bar{\alpha}_t \triangleq \prod_{i=1}^t\alpha_i$. Then we have the following result stated in Theorem~\ref{them:posterior}.

\begin{theorem}\label{them:posterior}
    In the skipped-step diffusion model defined above, the forward noising and posterior distributions, $q(\xb_t|\xb_{t-m})$ and $q(\xb_{t-m}|\xb_t, \xb_0)$, have the following forms:
    \begin{align*}
        q(\xb_t|\xb_{t-m}) &= \mathcal{N}\left(\xb_t; \sqrt{\frac{\Bar{\alpha}_t}{\Bar{\alpha}_{t-m}}}\xb_{t-m}, \left(1 - \frac{\Bar{\alpha}_t}{\Bar{\alpha}_{t-m}}\right)\Ib\right) \\
        q(\xb_{t-m}|\xb_t, \xb_0) &= \mathcal{N}\left(\xb_{t-m}; \frac{\sqrt{\Bar{\alpha}_t}(1 - \Bar{\alpha}_{t-m}) }{\sqrt{\Bar{\alpha}_{t-m}}(1 - \Bar{\alpha}_t)}\xb_t \right. \\
        & \hspace{-1cm}\left.+ \frac{(\Bar{\alpha}_{t-m} - \Bar{\alpha}_t)}{\sqrt{\Bar{\alpha}_{t-m}}(1 - \Bar{\alpha}_t)}\xb_0, \frac{(\Bar{\alpha}_{t-m} - \Bar{\alpha}_t)(1 - \Bar{\alpha}_{t-m})}{\Bar{\alpha}_{t-m}(1 - \Bar{\alpha}_t)}\Ib\right)~.
    \end{align*}
\end{theorem}

\begin{proof}[Sketch Proof]
    According to the standard diffusion forward process, we have
    \begin{align*}
        \xb_t &= \sqrt{\alpha_t}\xb_{t-1} + \sqrt{1 - \alpha_t}\epsilon
        = \sqrt{\alpha_t\cdots\alpha_{t-m+1}}\xb_{t-m} \\
        &~~~~+ \sqrt{1 - \alpha_t\cdots\alpha_{t-m+1}}\epsilon \\
        &= \sqrt{\frac{\Bar{\alpha}_t}{\Bar{\alpha}_{t-m}}}\xb_{t-m} + \sqrt{1 - \frac{\Bar{\alpha}_t}{\Bar{\alpha}_{t-m}}}\epsilon~.
    \end{align*}
    Thus, we have: $q(\xb_t|\xb_{t-m}) = \mathcal{N}\left(\xb_t; \sqrt{\frac{\Bar{\alpha}_t}{\Bar{\alpha}_{t-m}}}\xb_{t-m}, \left(1 - \frac{\Bar{\alpha}_t}{\Bar{\alpha}_{t-m}}\right)\Ib\right)$.

    In addition, we have $q(\xb_{t-m}|\xb_t, \xb_0) = \frac{q(\xb_t|\xb_{t-m})q(\xb_{t-m}|\xb_0)}{q(\xb_t|\xb_0)}$, and according to the standard diffusion, $q(\xb_{t-m}|\xb_0) = \mathcal{N}\left(\xb_{t-m}; \sqrt{\Bar{\alpha}_{t-m}}\xb_0, (1 - \Bar{\alpha}_{t-m})\Ib\right)$, $q(\xb_t|\xb_0) = \mathcal{N}\left(\xb_t; \sqrt{\alpha}_t\xb_0, (1 - \Bar{\alpha}_t)\Ib\right)$. By some simplification, we reach the posterior distribution $q(\xb_{t-m}|\xb_t, \xb_0)$ in the form of what the theorem states. 
\end{proof}

According to the standard diffusion loss in \eqref{eq:obj}, the next step is to parameterize a reverse model $p_{\thetab}(\xb_{t-m}|\xb_t)$ to match the skipped-step posterior distribution $q(\xb_{t-m}|\xb_t, \xb_0)$ for each time step $t$. From Theorem~\ref{them:posterior}, the forward posterior still follows a Gaussian distribution. Thus, it is natural parameterize the reverse model with Gaussian distributions as well. Consequently, matching the two posterior distributions as defined in \eqref{eq:obj} translates to matching the means of the two Gaussians. To this end, note that from $q(\xb_t|\xb_0)$, we have
\begin{align}\label{eq:predx0}
    \xb_0 = \frac{1}{\sqrt{\Bar{\alpha}_t}}\xb_t - \frac{\sqrt{1 - \Bar{\alpha}_t}}{\sqrt{\Bar{\alpha}_t}}\epsilon~.
\end{align}
Thus, by substituting \eqref{eq:predx0} into the mean of $q(\xb_{t-m}|\xb_t, \xb_0)$ derived in Theorem~\ref{them:posterior}, the forward posterior mean can be reformulated as
\begin{align}\label{eq:forwardposter}
    \mu_t &= \frac{\sqrt{\Bar{\alpha}_t}(1 - \Bar{\alpha}_{t-m})\xb_t + (\Bar{\alpha}_{t-m} - \Bar{\alpha}_t)(\frac{1}{\sqrt{\Bar{\alpha}_t}}\xb_t - \frac{\sqrt{1 - \Bar{\alpha}_t}}{\sqrt{\Bar{\alpha}_t}}\epsilon)}{\sqrt{\Bar{\alpha}_{t-m}}(1 - \Bar{\alpha}_t)} \nonumber \\
    &= \frac{\Bar{\alpha}_{t-m}}{\sqrt{\Bar{\alpha}_t\Bar{\alpha}_{t-m}}}\xb_t - \frac{\Bar{\alpha}_{t-m} - \Bar{\alpha}_t}{\sqrt{\Bar{\alpha}_t\Bar{\alpha}_{t-m}(1 - \Bar{\alpha}_t)}}\epsilon~.
\end{align}
% Thus, we can simply parameterize the reverse denoise network to predict the noise $\epsilon$ of the forward process, similar to standard diffusion model parametrization. The only difference is that, in the training, we do the following in each iteration:
% \begin{itemize}
%     \item Sample $t\sim [1, T]$, and $m\sim [1, \max(1, t-1)]$.
%     \item Sample $\xb_{t-m}|\xb_0 \sim q(\xb_{t-m}|\xb_0)$, and then $\xb_t|\xb_{t-m} \sim q(\xb_t|\xb_{t-m})$.
% \end{itemize}
Based on the above posterior mean formula in \eqref{eq:forwardposter}, we can simply parameterize the reverse denoise network to predict the noise $\epsilon$ of the forward process as $\epsilon_{\thetab}\left(\xb_t, t\right)$, similar to standard diffusion model parametrization. Specifically, let $p_{\thetab}(\xb_{t-m}|\xb_t) = \mathcal{N}(\xb_{t-m}; \mu_{\thetab}(\xb_t, t, m), \sigma_t^2)$ and $$\mu_{\thetab}(\xb_t, t, m) \triangleq \frac{\Bar{\alpha}_{t-m}}{\sqrt{\Bar{\alpha}_t\Bar{\alpha}_{t-m}}}\xb_t - \frac{\Bar{\alpha}_{t-m} - \Bar{\alpha}_t}{\sqrt{\Bar{\alpha}_t\Bar{\alpha}_{t-m}(1 - \Bar{\alpha}_t)}}\epsilon_{\thetab}(\xb_t, t)~.$$
Matching $p_{\thetab}(\xb_{t-m}|\xb_t)$ and $q(\xb_{t-m}|\xb_t, \xb_0)$ with KL-divergence leads to the following iterative process for training the skipped-step reverse model:
\begin{itemize}
    \item Sample $t\sim [1, T]$; then sample $m\sim [1, \max\{1, t-1\}]$.
    \item Sample $\xb_{t}|\xb_0 \sim q(\xb_{t}|\xb_0)$.
    \item Sample noise $\epsilon\sim\mathcal{N}(0, \Ib)$.
    \item Optimize: 
    \begin{align}\label{eq:multistepobj}
        \min_{\thetab}J \triangleq \mathbb{E}_{\xb_0, \epsilon}\left[\frac{(\Bar{\alpha}_{t-m} - \Bar{\alpha}_t)^2}{2\sigma_t^2\Bar{\alpha}_t\Bar{\alpha}_{t-m}(1 - \Bar{\alpha}_t)}\left\|\epsilon - \epsilon_{\thetab}\left(\xb_t, t\right)\right\|^2\right]~.
    \end{align}
\end{itemize}

\begin{algorithm}[tb]
\caption{Skipped-Step Sampling in Diffusion Models}
\label{alg:skipsample}
\textbf{Input}: A pretrained diffusion model \\
% \textbf{Parameter}: Optional list of parameters\\
\textbf{Output}: A generated sample
\begin{algorithmic}[1] %[1] enables line numbers
\STATE Initialize $\xb\sim\mathcal{N}(0, \Ib)$.
\STATE Set $m$ to be some step interval. Set $t = T$.
\WHILE{$t > 0$}
\STATE $t^\prime = \max\{t-m, 0\}$; 
\STATE $\epsilon \sim \mathcal{N}(0, \Ib)$ if $t > m$, else $\epsilon = 0$.
\STATE $\xb^\prime = \frac{\Bar{\alpha}_{t^\prime}}{\sqrt{\Bar{\alpha}_t\Bar{\alpha}_{t^\prime}}}\xb - \frac{\Bar{\alpha}_{t^\prime} - \Bar{\alpha}_t}{\sqrt{\Bar{\alpha}_t\Bar{\alpha}_{t^\prime}(1 - \Bar{\alpha}_t)}}\epsilon_{\thetab}\left(\xb, t\right) + \sqrt{\frac{(\Bar{\alpha}_{t^\prime} - \Bar{\alpha}_t)(1 - \Bar{\alpha}_{t^\prime})}{\Bar{\alpha}_{t^\prime}(1 - \Bar{\alpha}_t)}}\epsilon$.
\STATE $t = t^\prime$; $\xb = \xb^\prime$.
\STATE Set $m$ to some value by following some specific user-defined scheme.
% \IF {conditional}
% \STATE Perform task A.
% \ELSE
% \STATE Perform task B.
% \ENDIF
\ENDWHILE
\STATE \textbf{return} $\xb_0$
\end{algorithmic}
\end{algorithm}

\begin{table*}[t!]
    \centering
    \caption{Quantitative comparison with state-of-the-art methods on ImageNet dataset. Our naive Skipped-Step sampling achieves higher IS compared to DDIM; while the mixed version consistently outperforms DDIM in terms of both IS and FID scores over various sampling step scenarios.}\label{fig:score}
    \scalebox{0.9}{
    \begin{tabular}{llc c c cc cc c c c }
    \toprule
    \multirow{2}{*}{\textbf{Method}}& 
    \multirow{2}{*}{\textbf{Dataset}}& 
   \multicolumn{2}{c}{\textbf{1000 steps}} & \multicolumn{2}{c}{\textbf{500 steps}} & \multicolumn{2}{c}{\textbf{100 steps}} & 
   \multicolumn{2}{c}{\textbf{50 steps}} & 
   \multicolumn{2}{c}{\textbf{25 steps}}  \\
      \cmidrule(lr){3-4} \cmidrule(lr){5-6} \cmidrule(lr){7-8}
      \cmidrule(lr){9-10}
      \cmidrule(lr){11-12}
        & 
   & IS&FID($\downarrow$) &IS&FID($\downarrow$) & IS&FID($\downarrow$)&
IS&FID($\downarrow$)&IS&FID($\downarrow$) \\

    \midrule
DDPM &\multirow{4}{*}{ImageNet}&  91.69&9.41&91.01&9.63&88.51&\textbf{10.02}& 83.42&\textbf{11.42}&69.44&16.75\\
	DDIM& &  82.34&10.88& 80.25&11.21&78.32&12.04&77.91&12.10&73.64&14.26\\
%\cmidrule(lr){2-12}
	Skipped-Step && \textbf{92.96} &9.92&88.26&10.59&82.11&14.55&72.27&18.43&65.65&20.83\\
    Mix& &92.72&\textbf{9.29}&\textbf{92.59}&\textbf{9.61}&\textbf{91.88}&10.35&\textbf{85.46}&11.78&\textbf{77.02}&\textbf{14.15}\\
    % \midrule
    % DDPM &\multirow{4}{*}{LSUN\_bedroom}& \\
    % DDIM& & 2.384&2.939\\
    % Skipped-Step && \\
    % Mix& & \\
    \bottomrule
    \end{tabular}
    }
    \vspace{-0.1cm}
\end{table*}

\begin{table*}[t!]
    \centering
    \caption{Quantitative comparison with state-of-the-art methods on LSUN\_bedroom dataset. We report the precision metric to measure the fraction of generated samples that lie on the data manifold, and the recall metric to measures the fraction of real data that can be matched by generated samples. }\label{fig:score2}
    \scalebox{0.78}{
    \begin{tabular}{llc c cccccc cc cc c c c }
    \toprule
    \multirow{2}{*}{\textbf{Method}}& 
    \multirow{2}{*}{\textbf{Dataset}}& 
   \multicolumn{3}{c}{\textbf{1000 steps}} & \multicolumn{3}{c}{\textbf{500 steps}} & \multicolumn{3}{c}{\textbf{100 steps}} & 
   \multicolumn{3}{c}{\textbf{50 steps}} & 
   \multicolumn{3}{c}{\textbf{25 steps}}  \\
      \cmidrule(lr){3-5} \cmidrule(lr){6-8} \cmidrule(lr){9-11}
      \cmidrule(lr){12-14}
      \cmidrule(lr){15-17}
        & 
   & FID($\downarrow$) &Prec& Rec&FID($\downarrow$)&Prec& Rec & FID($\downarrow$)&Prec& Rec&
FID($\downarrow$)&Prec& Rec&FID($\downarrow$)&Prec& Rec \\

    \midrule
    DDPM &\multirow{4}{*}{LSUN}& 2.44&0.66&0.54& 3.63&\textbf{0.70}&\textbf{0.72}&7.02&\textbf{0.68}&0.71&\textbf{7.82}&\textbf{0.66}&0.70& 11.75&0.59&0.66\\
    DDIM& & 2.58&0.65&0.52& 4.21&0.67&0.72&7.14&0.65&\textbf{0.72}&9.10&0.65&\textbf{0.71}&11.26&0.62&\textbf{0.69}\\
    Skipped-Step &&2.42&\textbf{0.68}&\textbf{0.57}&3.69&0.68& 0.69&6.55&0.65&0.67& 8.93&0.59&0.63&12.83&0.52&0.61 \\
    Mix& &\textbf{2.38}&0.67&0.53&\textbf{3.61}&0.68&0.71&\textbf{6.35}&\textbf{0.68}&0.70&7.88&\textbf{0.66}&0.69&\textbf{10.15}& \textbf{0.62}&0.67 \\

    \bottomrule
    \end{tabular}
    }
    \vspace{-0.1cm}
\end{table*}

The loss \eqref{eq:multistepobj} is the same as the one adopted in DDPM, except with different coefficients at each time $t$. However, we can adopt the ``simple loss'' idea in DDPM to remove the coefficients, leading to exactly the same loss function as DDPM. This indicates that training the skipped-step reverse model is the same as training the reverse model in DDPM. In other words, pretrained diffusion models are secretly skipped-step samplers. Algorithm~\ref{alg:skipsample} describes how to apply a pretrained diffusion model for skipped-step generation.

\begin{remark}
    There is a common practice of using a subset of the whole time indexes in DDPM sampling for generation in order to reduce sampling steps. Compared to our update rule, {\it i.e.},  $\xb^\prime$ in Algorithm~\ref{alg:skipsample}, the common practice adopts the original coefficients, which lacks a formal justification, leading to sub-optimal results. 
\end{remark}

\subsection{Enhancing Skipped-Step Inference with DDIM Integration}

In practical scenarios, we observe that using DDIM can markedly enhance generation quality, particularly when initializing with less noisy inputs. To leverage this advantage, we introduce a novel approach that integrates our skipped-step sampling method with DDIM. We define $t_c$ as a cutoff point. For steps taken after the cutoff point (i.e., $t\geq t_c$), we apply skipped-step sampling, and for the rest steps (i.e., $t<t_c$), we switch to DDIM for sampling. Specifically, we first apply our skipped-step sampler for $k_c$ iterations to produce a preliminary noisy generation. This intermediate result is then used as the input for DDIM, which is executed for an additional $T-k_c$ iterations.

The cutoff point parameters $k_c$ will determine the duration each sampling method is applied, which will be systematically ablated in our experiments. Our empirical findings indicate that this combined approach significantly improves the quality of the generated samples, achieving state-of-the-art results. This integration harnesses the strengths of both sampling techniques, effectively optimizing the generation process.

\begin{figure*}[t!]
% \vspace{-0.3cm}
\centering     
\includegraphics[width=\linewidth]{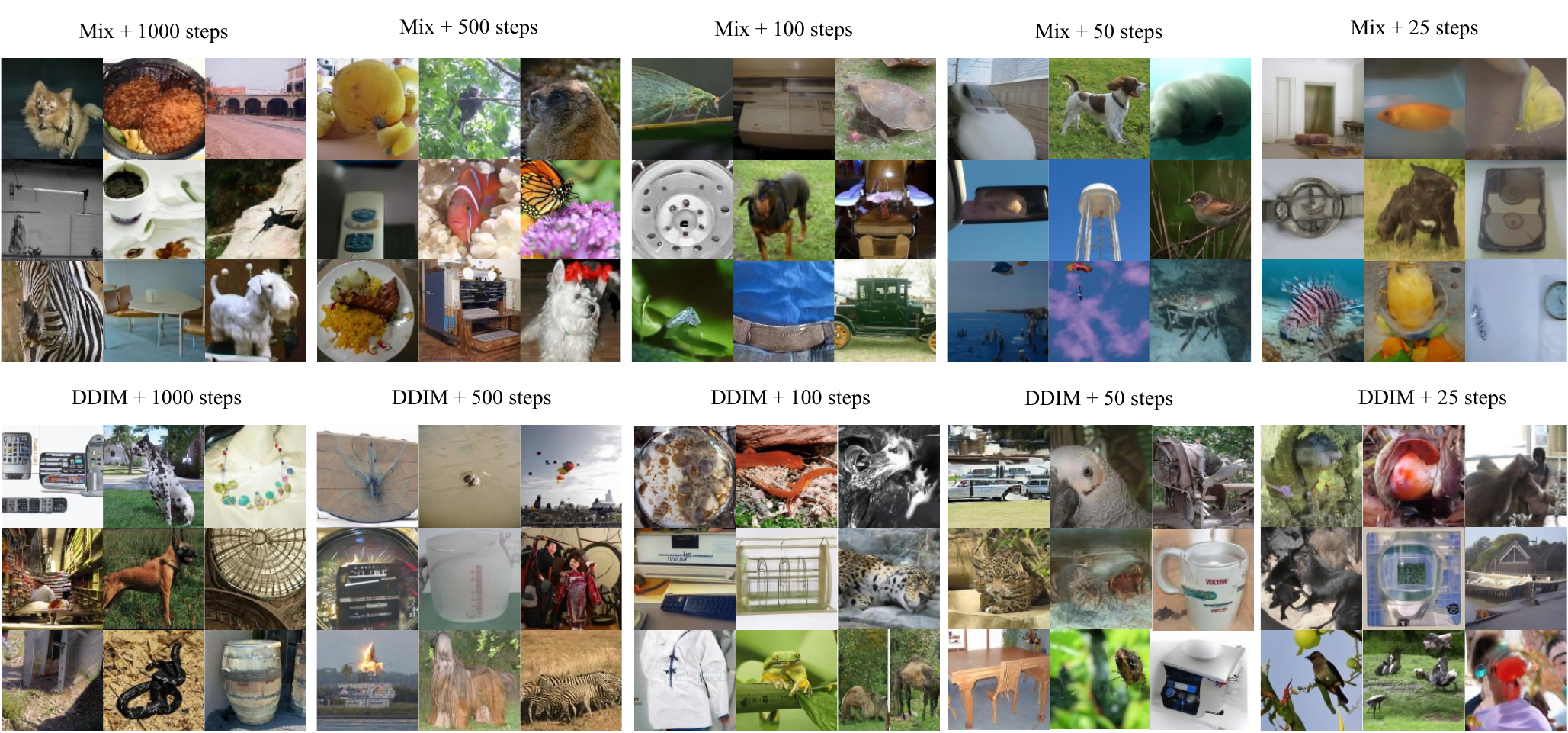}
%\vspace{-0.2cm}
\caption{Samples from the our mix version sampler and the DDIM sampler using different sampling steps.}
\label{fig:imagenet}
\vspace{-0.3cm}
\end{figure*}

\section{Related Works}
Since the development of diffusion models \cite{sohl2015deep,ho2020denoising,ren2024multi,li2023alleviating,xu2024disentangled, shen2025illumidiff}, there has been active research trying to address the inherent inefficiency of iterative denoising. Various acceleration techniques have been explored. \citet{song2021score} formulated denoising as solving a stochastic differential equation (SDE), which led to the development of score-based generative modeling. Their subsequent work \cite{song2020denoising} introduced an ODE-based formulation for faster sampling. Similarly, \citet{lu2022dpmsolver} developed DPM-Solver, a numerical method for efficiently solving diffusion ODEs, significantly reducing sampling steps. 
In parallel, the work by \citet{nichol2021improved} improved upon DDPMs by introducing a non-Markovian diffusion process, resulting in the Improved Denoising Diffusion Probabilistic Models (IDDPMs). Building on this, \citet{song2020denoising} proposed Denoising Diffusion Implicit Models (DDIM), which further enhanced sampling speed by using a non-Markovian process that retains high sample quality.

Research has also explored architectural optimizations and hardware accelerations. \citet{dhariwal2021diffusion} proposed enhanced neural network architectures tailored for diffusion models, balancing computational cost and generative performance. Additionally, guided sampling techniques, such as classifier guidance \cite{dhariwal2021diffusion}, have been shown to improve sample quality without significantly increasing inference time.

More recent works have focused on practical deployment and efficiency in real-world scenarios. \citet{xiao2021tackling} addressed the challenge of real-time inference by proposing techniques for reducing the number of required sampling steps while maintaining high fidelity in generated samples. \citet{watson2022accelerating} explored parallelization and optimization strategies for accelerating diffusion model inference on modern hardware.

In contrast to the aforementioned works, our proposed method tries to tackle the efficient inference problem from another perspective by discovering an intrinsic property of pretrained diffusion model, which directly allows skipped-step sampling. Our proposed method is orthogonal to existing methods, thus they can be seamlessly combined to get the best of both worlds.

\begin{figure}[t!]
% \vspace{-0.1cm}
\centering     
\includegraphics[width=\linewidth]{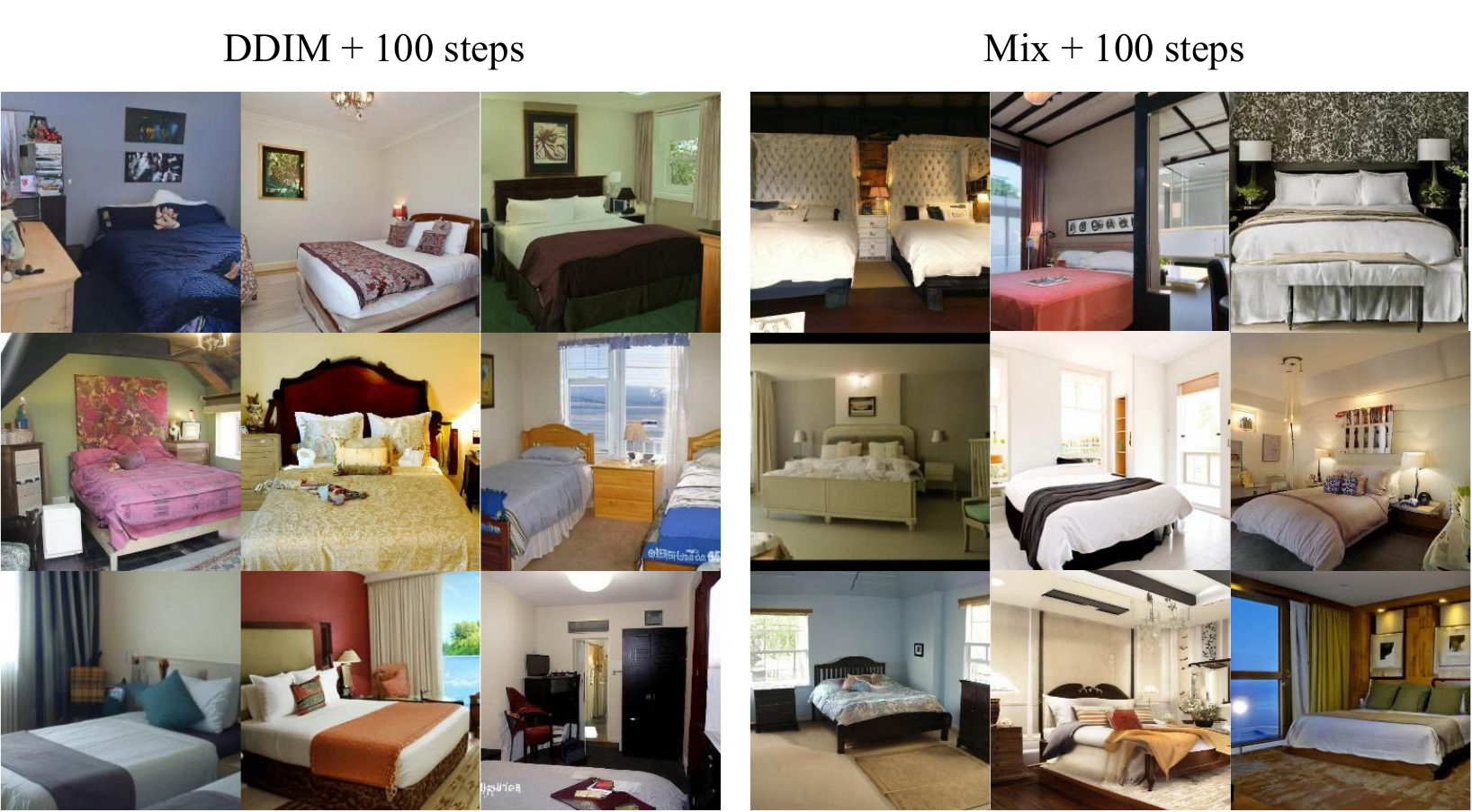}
\vspace{-0.3cm}
\caption{Samples from the our mix version sampler and the DDIM sampler using 100 sampling steps.}
\label{fig:lsun}
% \vspace{-0.2cm}
\end{figure}

% \begin{figure}[t!]
% % \vspace{-0.3cm}
% \centering     
% \includegraphics[width=\linewidth]{figures/overview_comparison.pdf}
% %\vspace{-0.2cm}
% \caption{Comparison of IS and FID score between the proposed method and DDIM sampler with the ADM model.}
% % \cy{Are you sure this is correct? Looks like DDIM has lower FIDs?}{\color{blue} Y axis is for FID. DDIM has higher FID scores}} \cy{From the figure, it is apprent that with the same IS (x-axis), DDIM has lower FID (y-axis)}
% \label{fig:fidis}
% \vspace{-0.3cm}
% \end{figure}

\begin{figure*}[t!]
    \centering
    \begin{subfigure}[t]{0.33\textwidth}
        \centering
\includegraphics[height=0.7\textwidth]{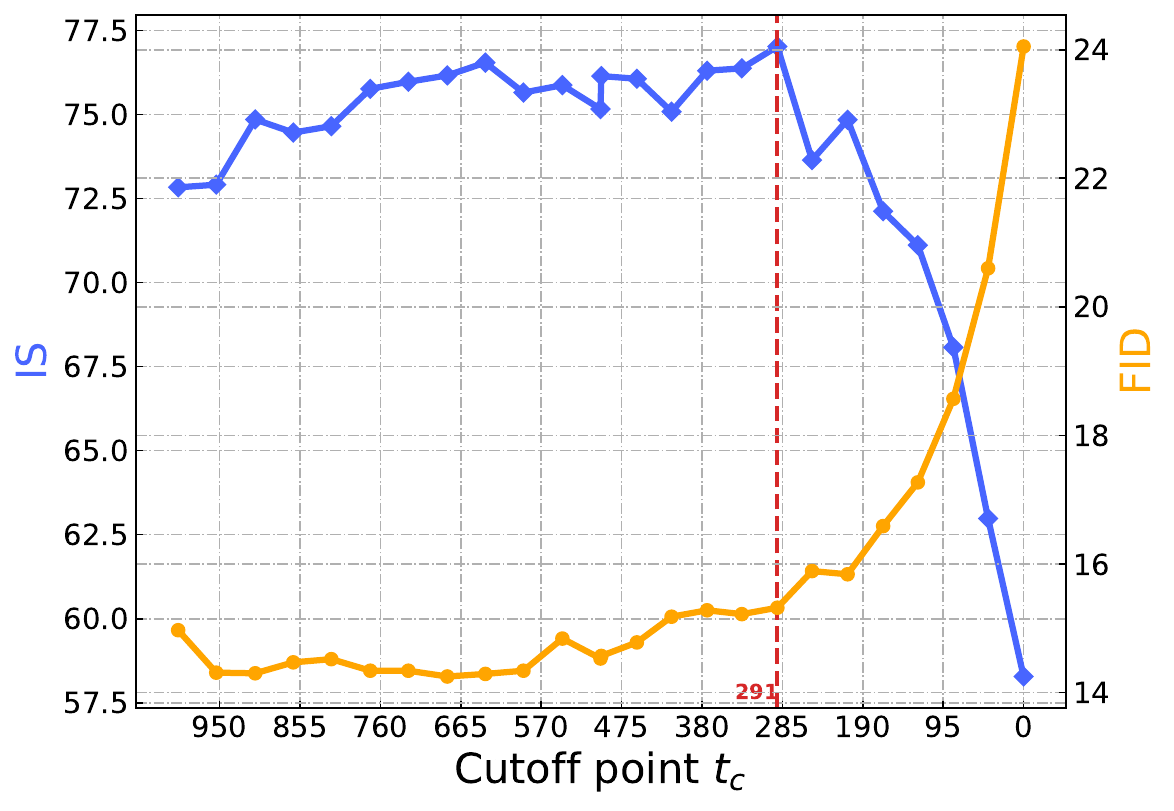}
        \caption{25 steps}
    \end{subfigure}
    \begin{subfigure}[t]{0.33\textwidth}
        \centering
\includegraphics[height=0.7\textwidth]{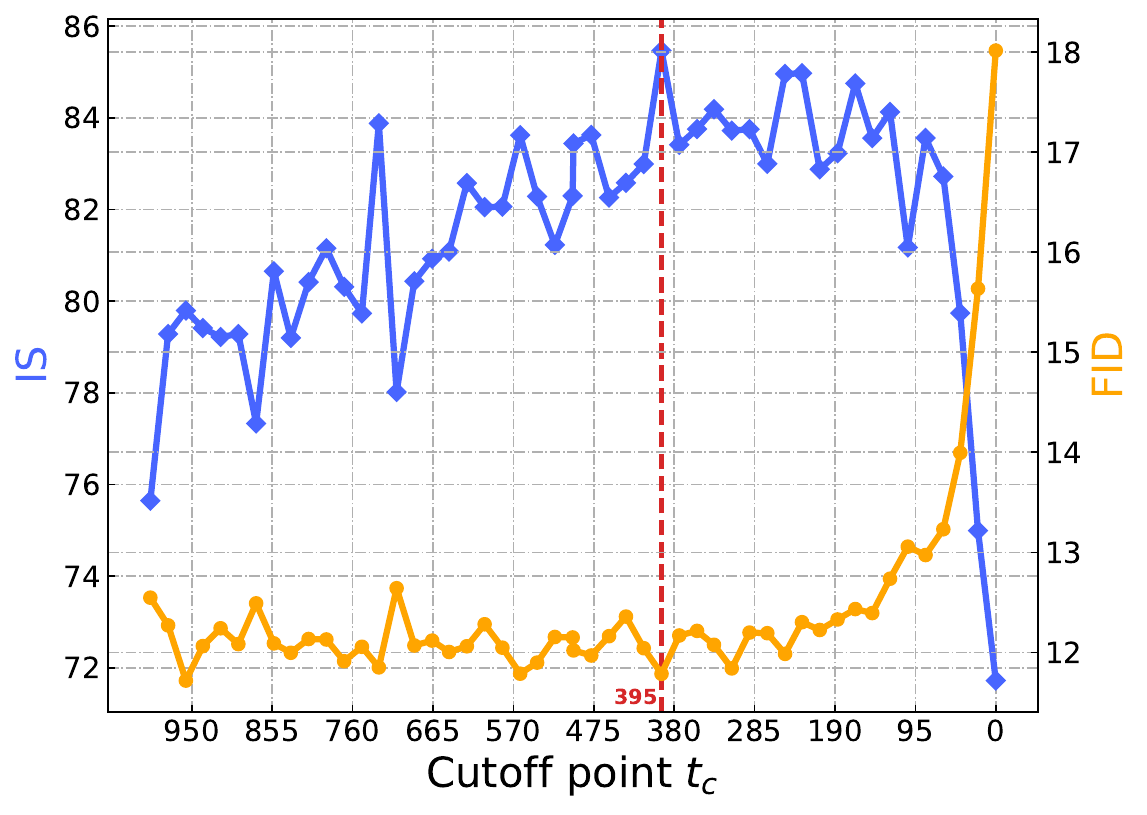}
        \caption{50 steps}
    \end{subfigure}
    \begin{subfigure}[t]{0.33\textwidth}
        \centering
\includegraphics[height=0.7\textwidth]{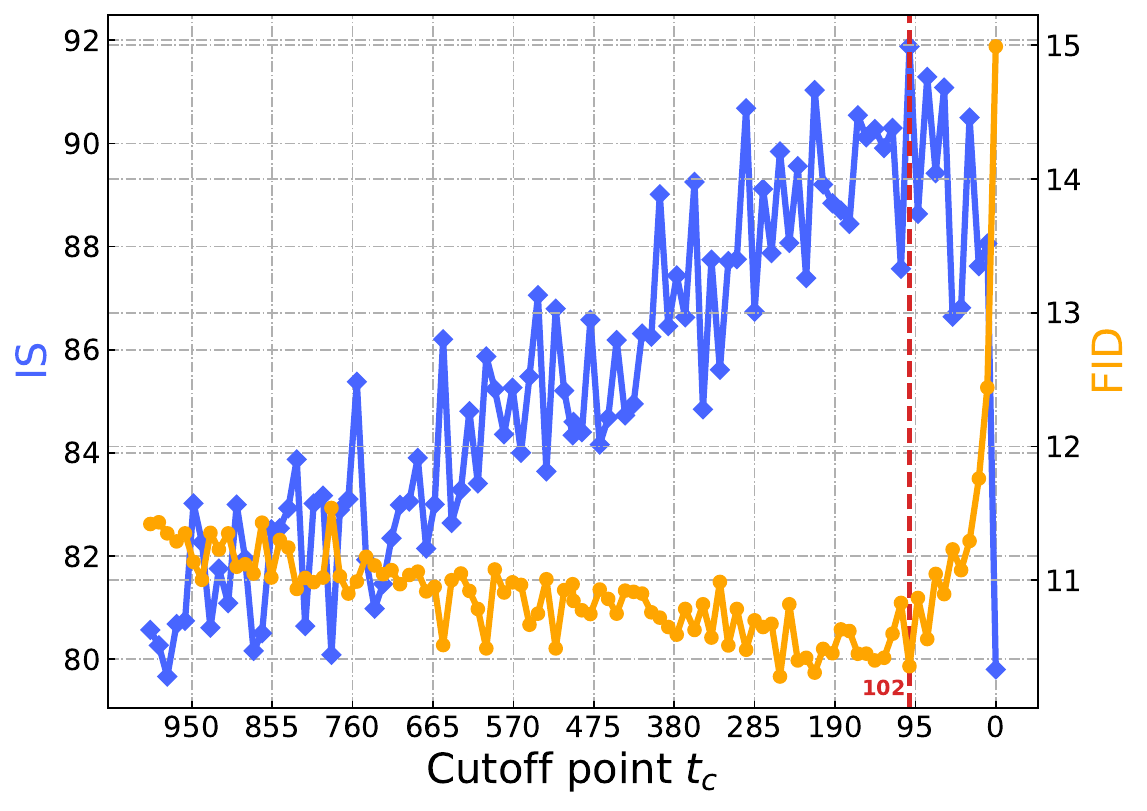}
        \caption{100 steps}
    \end{subfigure}
    \caption{Trade-off between FID (orange) and IS (blue) with different cut-off points under different inference steps.}
\label{fig:ablation_study}
\end{figure*}

% \begin{figure}[t!]
% % \vspace{-0.3cm}
% \centering     
% \includegraphics[width=\linewidth]{figures/mix.pdf}
% % \vspace{-0.3cm}
% \caption{Example of the generation process of skipped-step and its two mixed versions with DDIM on ImageNet. The horizontal black arrow represents the direction of the generation chain above it. (a) Given an image $x_0$, we generate noisy intermediate
% state $x_t$ according to Equation~\ref{eq:predx0}. We take the obtained $x_T$ as the initial noise to run the backward processes with (b) our skipped-step sampler; (c) mixed version with a cutoff point $t_c=200$; (d) mixed version with a cutoff point $t_c=600$.}
% \label{fig:chain}
% % \vspace{-0.4cm}
% \end{figure}

\section{Experiments}
% \cy{@Wenju: can you write about all the datasets we use. This is important for internal review.}

In this section, we conduct experiments to demonstrate that our method can be easily applied to various pretrained diffusion models, outperforming baseline methods where fewer sampling iterations are needed in generation. Throughout the paper, we use ``Skipped-Step'' to denote our naive version in Algorithm~\ref{alg:skipsample}, and use ``Mix'' to denote our enhanced version with DDIM.

\subsection{Experiment Setup}

\paragraph{Baseline Models} 
We test our proposed method on the pretrained ADM \cite{dhariwal2021diffusion}, Stable Diffusion v2 (SD2.1-base) \cite{rombach2022high}, and Open Sora v1.1 \cite{opensora} models. Our method does not need to change the training procedure. Thus, we reuse the pretrained base models trained with $T=1000$, and modify the sampling procedure in the public code to reflect our method. We compare our method with the default DDPM and DDIM samplers implemented in the codebases.

\paragraph{Dataset} 
We adopt the same evaluation datasets for the aforementioned pretrained models. Specifically, we sample the ADM trained on ImageNet \cite{deng2009imagenet} and LSUM\_bedroom \cite{yu2015lsun} datasets, respectively. And we generate evaluation image samples from Stable Diffusion with 30k captions sampled from the MS COCO 2014 eval split dataset \cite{lin2014microsoft}. To verify the performance of employing our sampling method to OpenSora, we conduct video generation based on the VBench benchmark \cite{huang2023vbench}.

\paragraph{Evaluation Metrics}
For image generation, we follow standard setup to adopt the inception score (IS) and FID score to measure quality and diversity of generated images \cite{InceptionFID,heusel2017gans,salimans2016improved}. On LSUN\_bedroom dataset, we report the Recall and Precision scores \cite{kynkaanniemi2019improved}. For video generation, we follow open Sora to calculate the quality score and semantic score to measure the quality of the generated videos \cite{opensora}.

% \cy{Please specify how you set the sampling steps in the ``mixed'' version, for all the 3 models.}

\begin{figure*}[ht!]
% \vspace{-0.1cm}
\centering     
\includegraphics[width=\linewidth]{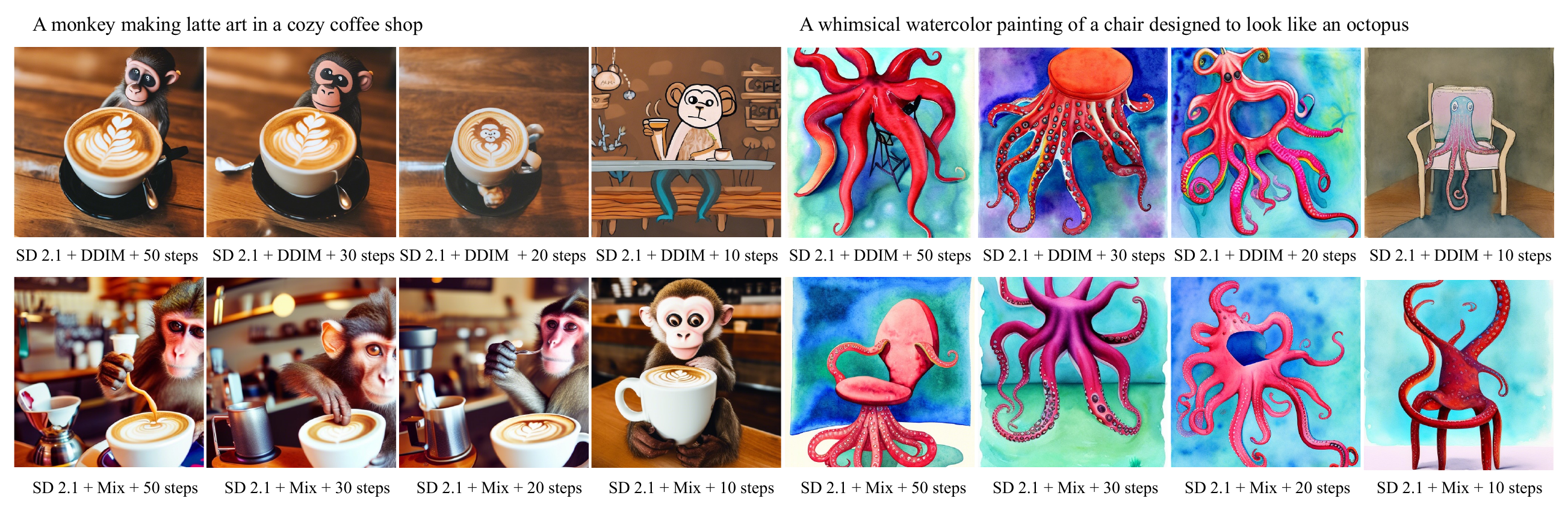}
% \vspace{-0.3cm}
\caption{Example comparisons of our method and DDIM sampling on pretrained stable diffusion models. Visually, our method can generate images that align better with the prompts.}
\label{fig:sd_ex}
% \vspace{-0.2cm}
\end{figure*}

\subsection{Experiments on ADM}
We first test our sampling method based on the pretrained ADM to demonstrate the ability of our method in improving image generation with fewer sampling steps. We consider two variants of our method: the vanilla Skipped-Step version as demonstrated in Algorithm~\ref{alg:skipsample}, and the improved version by mixing with DDIM, where we apply $k_c$ skipped-step sampling to generate an initial noisy input for the rest of steps with DDIM sampling. 

We vary the total sampling steps, each corresponds to a different IS-FID pair. 
We list the specific scores under different sampling steps in Table~\ref{fig:score}. It is observed that DDPM can outperform DDIM when the sampling steps are set relatively large. Our vanilla skipped-step sampling version can have better IS scores than DDIM. Remarkably, when mixing our vanilla version with DDIM, we obtain the best performance in terms of both IS and FID scores, under different sampling steps. Figure~\ref{fig:imagenet} and Figure~\ref{fig:lsun} demonstrate the qualitative comparison between our proposed method and DDIM sampler. It is clear that our method created images with better visual quality under different sampling steps. 

\begin{table}[t!]
    \centering
    \caption{Quantitative comparison with stable diffusion for image generation in terms of IS and FID score. Our method obtains better IS and FID scores over various sampling steps.}\label{tab:score_sd}
    \scalebox{0.74}{
    \begin{tabular}{lcccccccc}
    \toprule
	 \multirow{2}{*}{\textbf{Method}}& \multicolumn{2}{c}{\textbf{50 steps}} & \multicolumn{2}{c}{\textbf{30 steps}} &  \multicolumn{2}{c}{\textbf{20 steps}}&  \multicolumn{2}{c}{\textbf{10 steps}}  \\
      \cmidrule(lr){2-3} \cmidrule(lr){4-5} \cmidrule(lr){6-7}
      \cmidrule(lr){8-9}
   &  IS&FID($\downarrow$)&IS&FID($\downarrow$)&IS&FID($\downarrow$)&IS&FID($\downarrow$) \\
    \midrule
	SD\_2.1 &39.11& 17.48&38.15&19.23&37.31&20.23&30.58&25.64\\
    Mix &39.14&16.97&38.39&18.16&38.01&19.13&31.76&24.29\\  
    \bottomrule
    \end{tabular}
    }
    \vspace{-0.1cm}
\end{table}

% \paragraph{Comparison of Generation Chain} Experimental results show that our skipped-step sampler can improve the quality of generated images and its mixed version can further enhance the image generation. We present the comparison of the generation chain of the skipped-step method and two mixed versions of our samplers in Figure~\ref{fig:chain}. Lets take the generation chain of Figure~\ref{fig:chain}(b) as a reference. With $t_c = 200$, switching to DDIM will not modify the content but improve the image quality since $x_{t_c}$ already shows clear patterns and modes. By comparison, switching to DDIM when $t_c = 600$ will lead to an image with different content, instead of improving the quality of the generated image. This observation led us to choose a proper cutoff point, where the skipped-step sampler provided a reasonable initial noise.

\begin{figure}[ht!]
% \vspace{-0.1cm}
\centering     
\includegraphics[width=\linewidth]{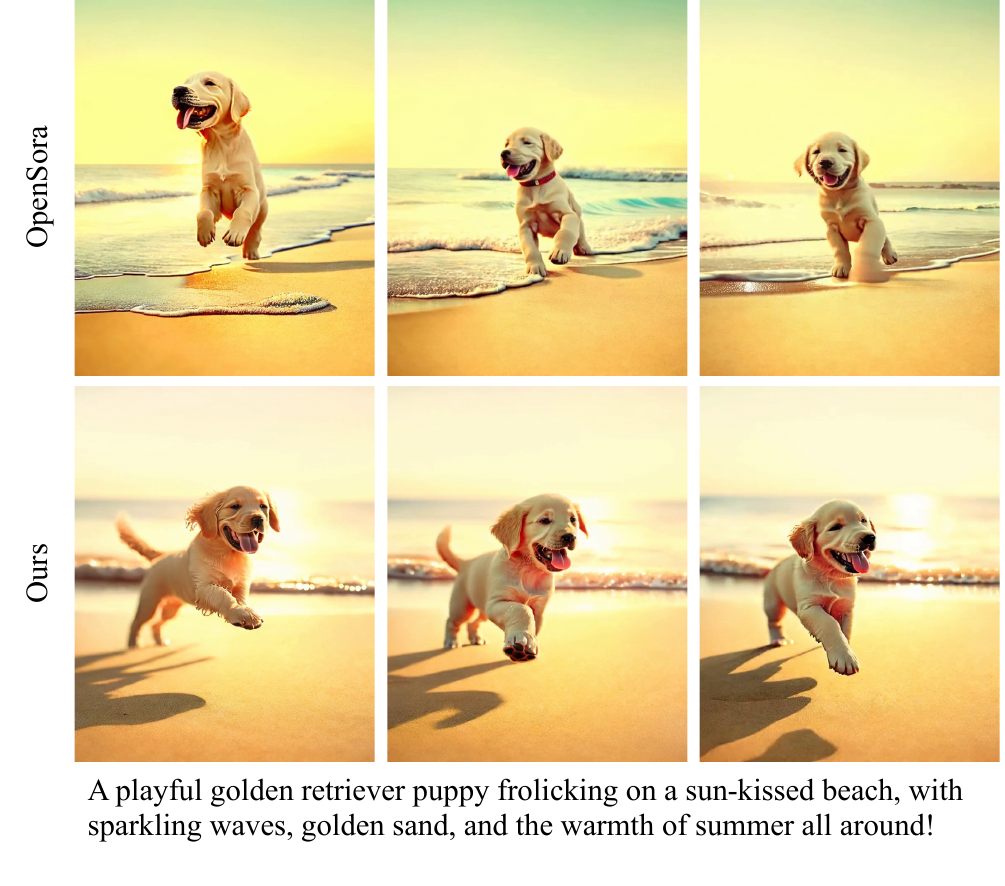}
\vspace{-0.3cm}
\caption{Example generation comparisons of Open Sora and our method. Our method generates videos that are better in both video quality and prompt alignments.}
\label{fig:sora}
% \vspace{-0.2cm}
\end{figure}

\paragraph{Ablation Study on DDIM Integration} To validate the benefits of integrating DDIM into our skipped-step sampler, we conducted an ablation study comparing the performance of different cutoff points $t_c$ in our mixed-version sampler. Figure~\ref{fig:ablation_study} presents the metric scores (FID and IS) achieved by various configurations of our sampler using different step sizes. The results highlight that integrating DDIM significantly enhances image generation quality, as evidenced by improved FID and IS scores. Furthermore, we observed that the optimal range for cutoff values lies within [100, 400]. Specifically, when sampling with a total of 25 steps, a smaller cutoff point (e.g., 291) is preferred compared to scenarios involving more DDIM sampling steps. A potential explanation is that the increased skipped-step sampling steps contribute to a more stable initial prediction, allowing the DDIM sampling phase to better refine the quality of the generated images.

\subsection{Experiments on Stable Diffusion}
Next, we test our method on the more scalable stable diffusion model (SD v2.1). Similarly, we replace the DDIM sampler in the codebase with our method and leave the other parts unchanged. For quantitative evaluation, we follow existing work to promote the pretrained stable diffusion with 30k captions from the MS COCO dataset, which are fed to the stable diffusion for text-to-image generation. Table~\ref{tab:score_sd} summarizes the results of our method (the best performing mixing-with-DDIM version) compared with the DDIM sampler in terms of IS and FID scores under different sampling steps. Consistently, our method achieves higher IS and lower FID scores, demonstrating the advantage of our method to achieve higher generative image quality and better diversity.
In addition, Figure~\ref{fig:sd_ex} plots some generated examples from both our method and the baseline DDIM sampler. It is interesting to see that our method tends to generate images that can better align with the prompts, {\em e.g.}, in the second example, all generated images align with the ``chair designed to look like an octopus'' tune specified in the prompt.

\begin{figure}[ht!]
% \vspace{-0.1cm}
\centering     
\includegraphics[width=\linewidth]{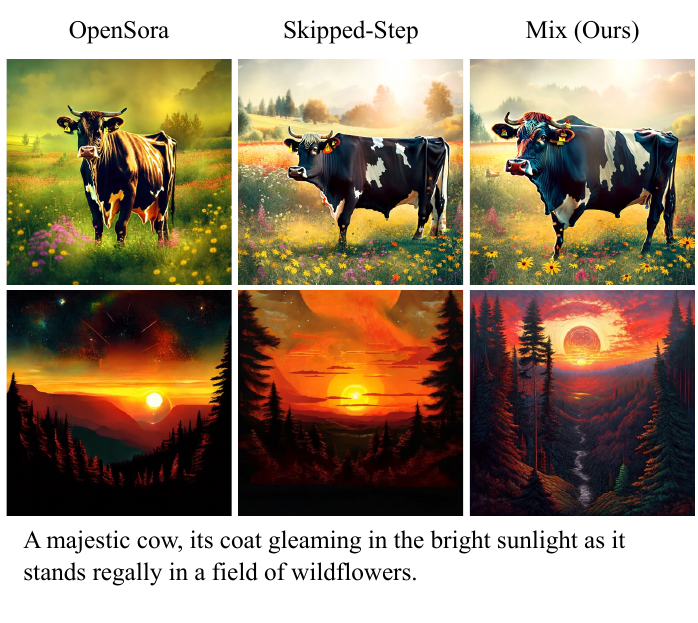}
\vspace{-0.3cm}
\caption{Example generation comparisons of Open Sora and our method. Our method generates videos that are better in both video quality and prompt alignments.}
\label{fig:sora}
% \vspace{-0.2cm}
\end{figure}

\begin{table}[t!]
    \centering
    \caption{Quantitative comparison with OpenSora for video generation. Our method obtains better results in all score metrics.}
    \label{tab:sora}
    \scalebox{0.9}{
    \begin{tabular}{lccccc}
    \toprule
   \textbf{Method}&\textbf{Total Score}& \textbf{Quality Score}&\textbf{Semantic Score}  \\
    \midrule
	OpenSora\_1.1 &74.26\%& 77.30\%&62.12\%\\
    Mix &74.66\%&77.51\%&63.27\%\\ 
    \bottomrule
    \end{tabular}
    }
    \vspace{-0.1cm}
\end{table}

\subsection{Experiments on Open Sora}
Finally, we test our method on the recent trending topic of video generation. We adopt one of the open-source state-of-the-art baselines -- open Sora, which adopts diffusion models for video generation \cite{opensora}. Again, we implement our sampling method inside the codebase and leave other settings unchanged. For quantitative evaluation, we adopt the same evaluation metrics as in open Sora baseline, which adopts the quality score, semantic score, and the total score, to measure both the generated quality and semantics of generated videos. 
Table~\ref{tab:sora} summarizes the results of our method and the DDIM baseline adopted in open Sora. Consistently, our method outperforms the baseline in all metrics. For visual illustration, Figure~\ref{fig:sora} plots some example generated videos with both our method and the DDIM baseline. Similar to the stable diffusion case, our method can generate higher-quality videos that tend to align better to the prompts, {\it e.g.}, with the prompt `` a small cactus with a happy face in the Sahara desert'', our method indeed can generate cactus with a face, whereas the DDIM sampler fails.

\section{Conclusion}
We discover an intrinsic property of pretrained diffusion models, which allows skipped-step sampling for accelerated generation in diffusion models for free. Our method is orthogonal to existing acceleration methods, which allows direct integration with these methods to achieve the best of both worlds. Extensive experiments on different scenarios including image generation and video generation demonstrate the effectiveness of our proposed method, achieving improved results compared to strong baselines. There are a number of interesting future directions worth further exploring. For example, how to incorporate the technique into training to further accelerate diffusion sampling, and how to incorporate the technique with other state-of-the-art diffusion models such as the consistency model \cite{Song2023ConsistencyM}.

% \section{Acknowledgments}
% AAAI is especially grateful to Peter Patel Schneider for his work in implementing the original aaai.sty file, liberally using the ideas of other style hackers, including Barbara Beeton. We also acknowledge with thanks the work of George Ferguson for his guide to using the style and BibTeX files --- which has been incorporated into this document --- and Hans Guesgen, who provided several timely modifications, as well as the many others who have, from time to time, sent in suggestions on improvements to the AAAI style. We are especially grateful to Francisco Cruz, Marc Pujol-Gonzalez, and Mico Loretan for the improvements to the Bib\TeX{} and \LaTeX{} files made in 2020.

% The preparation of the \LaTeX{} and Bib\TeX{} files that implement these instructions was supported by Schlumberger Palo Alto Research, AT\&T Bell Laboratories, Morgan Kaufmann Publishers, The Live Oak Press, LLC, and AAAI Press. Bibliography style changes were added by Sunil Issar. \verb+\+pubnote was added by J. Scott Penberthy. George Ferguson added support for printing the AAAI copyright slug. Additional changes to aaai25.sty and aaai25.bst have been made by Francisco Cruz and Marc Pujol-Gonzalez.

% \bigskip
% \noindent Thank you for reading these instructions carefully. We look forward to receiving your electronic files!

\newpage
\bibliographystyle{named}
\bibliography{reference}

\end{document}